\documentclass{article}

\usepackage[preprint,nonatbib]{neurips_fitml_2024}
\usepackage[table]{xcolor} 
\usepackage[utf8]{inputenc} 
\usepackage[T1]{fontenc}    
\usepackage{hyperref}       
\usepackage{url}            
\usepackage{booktabs}       
\usepackage{amsfonts}       
\usepackage{nicefrac}       
\usepackage{microtype}      
\usepackage{xcolor}         
\usepackage{amsmath}
\usepackage{amsthm}
\usepackage{graphicx}
\usepackage{multirow}
\usepackage{subcaption}
\usepackage{capt-of}
\usepackage{varwidth}

\newtheorem{theorem}{Theorem}

\newtheorem{corollary}{Corollary}
\definecolor{lightblue}{rgb}{0.9, 0.95, 1.0}
\definecolor{mygreen}{rgb}{0.01, 0.5, 0.01}
\definecolor{myred}{rgb}{0.8, 0.01, 0.01}
\title{FourierKAN outperforms MLP on \\Text Classification Head Fine-tuning}

\author{Abdullah Al Imran\thanks{\ \ Equal contribution}\\
  University of Liverpool\\
  \texttt{a.al-imran@liverpool.ac.uk} \\\And
  Md Farhan Ishmam\footnotemark[1]  \\
  Islamic University of Technology \\
  \texttt{farhanishmam@iut-dhaka.edu} \\
  }

\author{%
  Abdullah Al Imran\thanks{Equal Contribution} \\
  University of Liverpool\\
  \texttt{a.al-imran@liverpool.ac.uk}\\
  \And
  Md Farhan Ishmam\footnotemark[1] \\
  Islamic University of Technology \\
  \texttt{farhanishmam@iut-dhaka.edu}\\
}

\begin{document}

\maketitle

\begin{abstract}
In resource constraint settings, adaptation to downstream classification tasks involves fine-tuning the final layer of a classifier (i.e. classification head) while keeping rest of the model weights frozen. Multi-Layer Perceptron (MLP) heads fine-tuned with pre-trained transformer backbones have long been the de facto standard for text classification head fine-tuning. However, the fixed non-linearity of MLPs often struggles to fully capture the nuances of contextual embeddings produced by pre-trained models, while also being computationally expensive. In our work, we investigate the efficacy of KAN and its variant, Fourier KAN (FR-KAN), as alternative text classification heads. Our experiments reveal that FR-KAN significantly outperforms MLPs with an average improvement of 10\% in accuracy and 11\% in F1-score across seven pre-trained transformer models and four text classification tasks. Beyond performance gains, FR-KAN is more computationally efficient and trains faster with fewer parameters. These results underscore the potential of FR-KAN to serve as a lightweight classification head, with broader implications for advancing other Natural Language Processing (NLP) tasks.
\end{abstract}

\section{Introduction}

Classification head fine-tuning, also known as \emph{linear probing}, is a widely adopted strategy that involves training the final classification layer while the backbone model remains frozen. This approach allows efficient adaptation to downstream tasks especially in resource constraint settings \cite{gao2023tuning} and improved robustness in out-of-domain distributional shifts \cite{kumar2022fine}, compared to standard fine-tuning. In text classification tasks, Multi-Layer Perceptron (MLP) \cite{hornik1989multilayerMLP} classification heads are usually fine-tuned with pre-trained transformer backbones \cite{vaswani2017attention}. 

MLPs are fully connected or dense neural networks, used across domains including time series analysis \cite{liu2018time}, computer vision \cite{lecun1998gradient, krizhevsky2017imagenet}, and speech processing \cite{hannun2014deep}. MLP classifiers are able to capture the non-linearity and aggregate the high-dimensional contextualized embedding produced by the feature extractor to the fixed set of output classes. While it is undeniable that MLPs have revolutionized deep learning, they have a few noticeable limitations \cite{muhlenbein1990limitations}.

MLPs account for most of the trainable parameters in the transformer architecture while being less interpretable compared to methods, \emph{e.g.} self-attention \cite{cunningham2023sparse}. Reducing the parameter count while maintaining performance has been explored by various strategies, such as network pruning \cite{han2015learning}, and quantization \cite{jacob2018quantization}, but with significant tradeoffs. Recent advancements in MLP alternatives, such as Kolmogorov-Arnold Networks (KANs) \cite{liu2024kan}, show promising results to replace MLPs by \emph{learning the non-linearity} \cite{dhiman2024kan} instead of relying on the fixed non-linear activations used in MLPs. 

Our work aims to investigate the adaptation of the KAN as a text classification head in resource-constraint settings by exploring its efficacy in linear probing. We primarily focus on the potential application of Fourier-KAN (FR-KAN) \cite{xu2024fourierkan}, a modification of the spline-based KAN variant using the Fourier series. To the best of our knowledge, we are the first to adapt the KAN architecture as an MLP alternative in linear probing. We also observed that the simple modification of using FR-KANs instead of MLPs as the classification head resulted in an average increase of $10\%$ in accuracy and $11\%$ in F1 score across $7$ text classification datasets using pre-trained transformer backbones. 

\section{Methodology}

\subsection{Text Classification Head}

For the contextual embedding, $\mathbf{H}= f(\mathbf{x};\mathbf{\theta}_f)$, produced by the pre-trained language model $f$ where $\mathbf{x}$ is the input text embedding and $\theta_f$ is the frozen model parameters, we formulate the predicted answer class, 
\begin{equation}
    \hat{y} = \underset{\substack{c \in \mathcal{C}}}{\arg\max} \text{ Head}(\mathbf{H})
\end{equation}
where, $\text{Head}:\mathbb{R}^\textbf{H} \rightarrow \mathbb{R}^c$, $\text{Head} \in \{\text{MLP}, \text{KAN}, \text{FR-KAN}\}$, and $c$ represents an answer class from the answer class set $\mathcal{C}$. It should be noted that the pre-trained language model $f$ acts as a feature extractor only. We utilize the cross-entropy or the negative log-likelihood loss during the classifier head fine-tuning, mathematically expressed, 
\begin{equation}
    \label{eq:crossEntropyLoss}
    L(\hat{y}, y) = -\sum_{i=1}^{|\mathcal{C}|} y_i \log(\hat{y}_i)
\end{equation}

\subsection{Kolmogorov-Arnold Network (KAN)}
Following the formulation of KANs with arbitrary depth and width \cite{liu2024kan}, a single KAN layer is defined as:
\begin{equation}
\label{eq:kan}
\text{KAN}(\mathbf{H}) = f(\mathbf{H}) = \sum_{j=1}^{2n+1} \Phi_j \left( \sum_{i=1}^{n} \phi_{ij}(h_i) \right)
\end{equation}
where, $\phi_{ij}$ are univariate continuous functions mapping the input vector $x$, such that, $\phi_{ij}:[0,1] \rightarrow \mathbb{R}$ and $\Phi_j$ are learnable activation functions, such that, $\Phi_j:\mathbb{R} \rightarrow \mathbb{R}$. The KAN layer can be analogous to a 2-layer MLP where the first layer computes the inner sum $\sum_{i=1}^{n} \phi_{ij}(x_i)$ and the next layer applies and sums $\Phi_j$ to the previous layer output. The original implementation of KAN \cite{liu2024kan} follows a residual layer formulation of the learnable activation function:
\begin{equation}
\label{eq:phiBSpline}
    \phi_b(x) = w(b(x) + \text{spline}(x))
\end{equation}
where, the basis function $b(x)$ are defined as,
\begin{equation}
    b(x) = \text{silu}(x) = \frac{x}{1 + e^{-x}} 
\end{equation}
and the splines can be formulated as the weighted sum of B-splines,
\begin{equation}
\label{eq:spline}
    \text{spline}(x) = \sum_{i=1}^{G} c_i B_i(x)
\end{equation}
where, $c_i$ are trainable parameters and $G$ is the grid size.

\begin{theorem}
\label{th:convTheorem}
Assume with Fourier coefficients $a_{k}, b_{k}$ and grid size $G$, the Fourier series for the function $f(x)$ taking the form:
\begin{equation*}
     f_{G}(x) = \sum_{k=0}^{G} \left( 
    a_{k} \cdot \cos(kx) + b_{k} \cdot \sin(kx)  \right)
\end{equation*}
converges to a corresponding univariate function over a finite interval $[a,b]$ as $G \rightarrow \infty$, given the function is continuous. 
\end{theorem}

\begin{proof}
   
The convergence of the Fourier series to a univariate function can be proved via pointwise, uniform, or mean square (or $L^2)$ convergence. We are particularly interested in uniform convergence, which implies pointwise and mean square convergence. For pointwise convergence, Dirichlet's proof states $f_{G}(x)$ converges at points of continuities and takes the average value of $(f(d^{+})+f(d^{-}))/{2}$ when jump discontinuity is observed at $x=d$.

We generalize the Fourier coefficients $a_k$ and $b_k$ as $c_k$ and consider the Riemann-Lebesgue lemma on the Fourier coefficients \emph{i.e.} $c_k \rightarrow 0$ as $k \rightarrow 0$. From pointwise convergence, we state for $x \in [a,b]$, 
\begin{equation}
    |f(x) - f_G(x)| = \left| \sum_{k=G+1}^{\infty} a_k \cdot \cos(kx) + \sum_{k=G+1}^{\infty} b_k \cdot \sin(kx) \right|
\end{equation}
Since, $\cos(kx) \leq 1$ and $\sin(kx) \leq 1$,
\begin{equation}
|f(x) - f_G(x)| \leq \sum^{\infty}_{k=G+1}(|a_k| + |b_k|)
\end{equation}
To ensure uniform convergence, the sum of the trailing elements is required to be bounded. As the Fourier coefficients are square summable \emph{i.e.} $\sum_{k=1}^{\infty} |c_k|^2 < \infty$, we apply the Cauchy-Schwarz inequality:
\begin{equation}
\left(  
    \sum^\infty_{k=G+1} |c_k|
\right)^2
\leq \left(  
    \sum^\infty_{k=G+1} 1^2
\right)
\left(  
    \sum^\infty_{k=G+1}  |c_k|^2
\right)
\end{equation}
$|c_k|^2$ converges as the Fourier coefficients are square summable. Hence, the tail sum of $\sum^\infty_{k=G+1}  |c_k|^2$ can be arbitary small as $G \rightarrow \infty$ and the Fourier series converges uniformly to $f(x)$ on $[a,b]$.
\end{proof}

\begin{corollary}
\label{cr:gridError}
    As $G \rightarrow \infty$, the truncation error of the Fourier series, $E_G \rightarrow 0$.
\end{corollary}

\subsection{FouRier KAN (FR-KAN)}

The residual formulation using B-splines in Eq. \ref{eq:phiBSpline} can be replaced with the Fourier series following the convergence of the series up to $G$ terms in Theorem \ref{th:convTheorem}. Hence, the univariate continuous function in Eq. \ref{eq:kan} can be defined as:
\begin{equation}
    \label{eq:frKAN}
    \phi_f(x) = \sum_{k=0}^{G} \left( 
    a_{k} \cdot \cos(kx) + b_{k} \cdot \sin(kx)  \right)
\end{equation}
where, $a_k$ and $b_k$ are the trainable Fourier coefficients and $G$ is the grid size. 

\subsection{Multi Layer Perceptron (MLP)}
\label{sec:mlpClassifierHead}
We define $\text{MLP}_L$ as an $L$-layer perceptron with the trainable weights, $W_{L-1}$ and bias, $b_{L-1}$. We formulate a $1$-layer perceptron:
\begin{equation}
\label{eq:mlp1}
    \text{MLP}_1(\mathbf{H}) = \text{softmax}(W_0 \mathbf{H} + b_0)
\end{equation}
where the softmax function is defined as:
\begin{equation}
\label{eq:softmax}
\text{softmax}(z_i) = \frac{e^{z_i}}{\sum_{j=1}^{n} e^{z_j}}
\end{equation}
We also define a 2-layer perceptron:
\begin{align}
h_0 &= \sigma(W_0 \mathbf{H} + b_0)\\ 
\text{MLP}_2(\mathbf{H}) &= \text{softmax}(W_1 h_{0} + b_1) \label{eq:mlp2}
\end{align}
where, the non-linear sigmoid, $\sigma$ function is defined:
\begin{equation}
\label{eq:sigmoid}
    \sigma(x) = \frac{1}{1 + e^{-x}}
\end{equation}

\section{Experiments}

\begin{table*}[ht]
\centering

\centering
\resizebox{.99\textwidth}{!}{%
\begin{tabular}{lcccccccc}

\toprule 
\textbf{Dataset} & \textbf{Task Name}      & \textbf{Type} & \textbf{\#Classes} & \textbf{Avg. Len} & \textbf{Max Len} & \textbf{\#Train} & \textbf{\#Val} & \textbf{\#Test} \\ \midrule
AgNews           & News Classification & TC     & 4                  & 44                & 221              & 10.5k             & 2.25k & 2.25k \\
DBpedia          & Ontology Classification & TC & 14                 & 67                & 3841             & 10.5k             & 2.25k & 2.25k             \\
IMDb             & Movie Sentiment & SA      & 2                  & 292               & 3045             & 10.5k             & 2.25k & 2.25k            \\
Papluca          & Language Identification & LI & 20              & 111               & 2422                 & 10.5k             & 2.25k & 2.25k          \\
SST-5            & General Sentiment & SA      & 5                  & 103                 & 283               & 8.3k             & 1.78k     & 1.78k \\
TREC-50          & Question Classification & QC & 50                 & 11                & 39               & 4.17k             & 893         &  893  \\
YELP-Full        & Review Sentiment & SA      & 5                  & 179               & 2342             & 10.5k             & 2.25k & 2.25k       \\ \bottomrule  
\end{tabular}}

\caption{Statistics of the text classification datasets used in our work. Full form of the types -- SA: Sentiment Analysis, TC: Topic Classification, QC: Question Classification, LI: Language Identification.}

\label{tab:datasets}

\end{table*}

\subsection{Tasks and Datasets}
We chose four types of text classification tasks and seven datasets to evaluate our models. The overall statistics of the datasets are shown in Tab. \ref{tab:datasets}.

\paragraph{Sentiment Analysis}

We use three datasets -- IMDb \cite{maas2011learning}, SST-5 \cite{socher2013recursive}, and Yelp-full \cite{zhang2015character} for sentiment analysis. The IMDb is a binary classification dataset on movie reviews, while SST-5 and YELP-full are multi-class classification datasets on movie reviews and general reviews, respectively.     

\paragraph{Topic Classification}
We use two datasets AgNews \cite{zhang2015character} and DBpedia \cite{zhang2015character} for topic classification. The AgNews dataset classifies news topics from over $2000$ news sources into $4$ topic classes. DBpedia introduces ontology classification as a form of topic classification on $14$ ontology classes, each with $40$k training samples and $5$k test samples.

\paragraph{Question Classification}
We use the $50$ class or fine-grained variant of the TREC dataset \cite{voorhees-tice-2000-trec} consisting of open-domain questions for question classification. 

\paragraph{Language Identification}
The Papluca dataset \cite{papluca_language_identification} classifies the text language into $20$ uniformly distributed classes.

\subsection{Models}
We utilize seven variants of pre-trained transformer models to generate the contextual embedding as seen in Tab. \ref{table:models}. BART \cite{lewis2019bart} is the only encoder-decoder model while BERT \cite{devlin2018bert}, DeBERTa \cite{he2020deberta}, DistilBERT \cite{sanh2019distilbert}, ELECTRA \cite{clark2020electra}, RoBERTa \cite{liu2019roberta}, and XLNet \cite{yang2019xlnet} are encoder-only models. BART has $12$ layers encoder and $12$ decoder layers, while the other models have $12$ encoder layers with the exception of DistilBERT with $6$ encoder layers.

\begin{table*}[ht]
\centering
      \begin{varwidth}[b]{0.5\linewidth}
\centering
\resizebox{0.79\textwidth}{!}{%
\begin{tabular}{lccc}
\toprule
\textbf{Model} & \multicolumn{1}{c}{\textbf{Arch}} & \multicolumn{1}{c}{\textbf{\#L}} & \multicolumn{1}{c}{\textbf{\#PC (M)}} \\ \midrule
BART           & ED                                 & 12+12                               & 140                                  \\
BERT           & E                                  & 12                               & 110                                  \\
DeBERTa        & E                                  & 12                               & 139                                  \\
DistilBERT     & E                                  & 6                                & 66                                   \\
ELECTRA        & E                                  & 12                               & 110                                  \\
RoBERTa        & E                                  & 12                               & 125                                  \\
XLNet          & E                                  & 12                               & 110                                   \\ \bottomrule
\end{tabular}}
\caption{Architecture type [E: Encoder, D: Decoder], number of layers, and parameter count of the transformer models.}
        \label{table:models}
      \end{varwidth}%
      \hspace{1cm}
      \begin{minipage}[b]{0.3\linewidth}
        \centering
        \includegraphics[width=\linewidth]{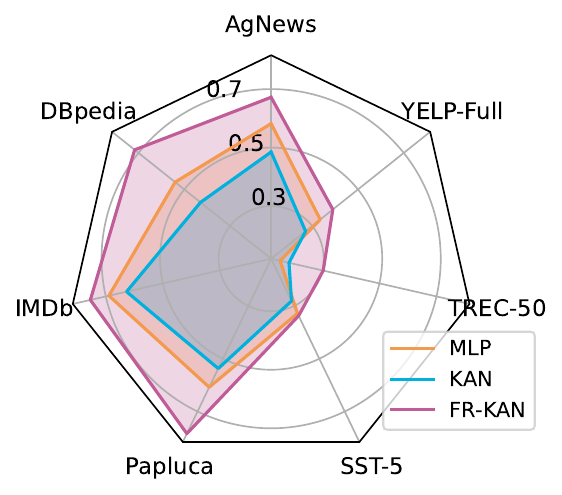}
        \captionof{figure}{Comparison of average accuracy of different classification heads.}
        \label{fig:radar}
      \end{minipage}
    \end{table*}

\subsection{Evaluation Metrics}
We evaluate the classification performance using four key metrics: accuracy, macro-averaged F1 score (simply F1 score), micro-averaged F1 score, and Cohen's kappa coefficient.

\subsection{Experimental Setup}
All models were fine-tuned on an A4000 GPU with 16 GiB of GPU memory. The default tokenizer and embedding layers corresponding to each model were used. Each model features a hidden dimension of $768$ and $12$ self-attention heads. The implementation and training configurations followed the HuggingFace library \cite{wolf-etal-2020-transformers}. The classifiers were fine-tuned using Adam optimizer with the max length set to $512$ and the batch size fixed at $64$. Identical training configurations were used across dataset-model pairs to ensure fair evaluation of the classification heads.

\subsection{Hyperparameters}
To ensure fairness of evaluation, all classification heads are defined as $1$-layer architectures, excluding the input layer. $2$ layer MLPs (Eq. \ref{eq:mlp2}) have also been later specified to align the number of trainable parameters with that of the other heads. The width of the hidden layer of MLP varied depending on the classification dataset. Unless specified, the original KAN and FR-KAN layers use a grid size of $1$ and $5$ respectively. To evaluate in resource-constraint settings, all the classifiers were fine-tuned for $5$ epochs at the learning rate of $2e-5$.

\begin{table*}[ht]
       \centering
      \begin{minipage}[b]{0.475\linewidth}
        \centering
        \includegraphics[width=0.91\linewidth]{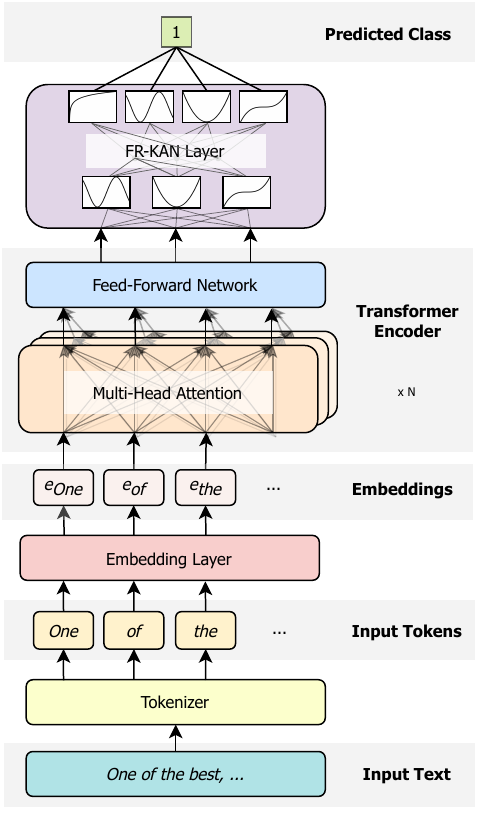}
        \captionof{figure}{Overview of the architecture with FR-KAN classification head -- following the standard tokenization and embedding, the input text is passed to a pre-trained transformer encoder. The FR-KAN layer maps the contextualized embedding produced by the transformer to the output classes.}
        \label{fig:archi}
      \end{minipage}
      \hfill
            \begin{varwidth}[b]{0.475\linewidth}
        \centering
        \resizebox{0.995\textwidth}{!}{
            \begin{tabular}{ccccc}
\toprule
\textbf{Dataset}             & \textbf{Classifier} & \textbf{\#PC(k) \textcolor{myred}{$\downarrow$}} & \textbf{Acc \textcolor{mygreen}{$\uparrow$}} & \textbf{F1 \textcolor{mygreen}{$\uparrow$}} \\ \midrule
\multirow{3}{*}{AgNews}  & MLP-40              & 30.9                  & 0.835        & 0.831       \\ 
& KAN-1 & 30.7 &0.812	& 0.813  \\ 
&\cellcolor{lightblue}FR-KAN-5               &\cellcolor{lightblue}30.7                  &\cellcolor{lightblue}\textbf{0.877}        &\cellcolor{lightblue}\textbf{0.876}       \\

       & Diff                & \textcolor{mygreen}{-0.2}                   & \textcolor{mygreen}{+0.04}        &  \textcolor{mygreen}{+0.05}       \\ \midrule
\multirow{3}{*}{Dbpedia}                           
                         & MLP-138             & 108.1                 & 0.892        & 0.891       \\ 
                         &KAN-1&107.5&0.843&0.842\\
                          &\cellcolor{lightblue}FR-KAN-5               &\cellcolor{lightblue}107.5                 & \cellcolor{lightblue}\textbf{0.970}        &\cellcolor{lightblue}\textbf{0.971}       \\
                         & Diff                & \textcolor{mygreen}{-0.6}                     &  \textcolor{mygreen}{+0.08}        &  \textcolor{mygreen}{+0.08}       \\ \midrule
\multirow{3}{*}{IMDb}                                
                         & MLP-20              & 15.4                  & 0.778        & 0.777       \\ 
                         & KAN-1 & 15.4 & 0.739 &0.739\\
                         &\cellcolor{lightblue}FR-KAN-5               &\cellcolor{lightblue}15.4                  & \cellcolor{lightblue}\textbf{0.831}        &\cellcolor{lightblue}\textbf{0.830}       \\
                         & Diff                & 0.0                   &  \textcolor{mygreen}{+0.05}        &  \textcolor{mygreen}{+0.05}      \\ \midrule 
\multirow{3}{*}{Papluca}                             
                         & MLP-196             & 154.6                 & 0.816        & 0.819       \\ 
                         & KAN-1 & 153.6 & 0.730 & 0.733\\
                          & \cellcolor{lightblue}FR-KAN-5               &\cellcolor{lightblue}153.6                 &\cellcolor{lightblue}\textbf{0.986}        &\cellcolor{lightblue}\textbf{0.986}       \\
                         & Diff                & \textcolor{mygreen}{-1.0}                     &  \textcolor{mygreen}{+0.17}        & \textcolor{mygreen}{+0.17}      \\ \midrule
\multirow{3}{*}{SST-5}        
                         & MLP-50              & 38.7                  & 0.351        & 0.176       \\ 
                         & KAN-1 & 38.4 & 0.307 & 0.231\\
                          &\cellcolor{lightblue}FR-KAN-5               &\cellcolor{lightblue}38.4                  &\cellcolor{lightblue}\textbf{0.401}        &\cellcolor{lightblue}\textbf{0.336}       \\
                         & Diff                & \textcolor{mygreen}{-0.3}                    & \textcolor{mygreen}{+0.05}       & \textcolor{mygreen}{+0.16 }       \\ \midrule
\multirow{3}{*}{TREC-50}          
                         & MLP-477             & 512.8                 & 0.179        & 0.006       \\ 
                         & KAN-1 & 384 & 0.188	& 0.017 \\
                         &\cellcolor{lightblue}FR-KAN-5               &\cellcolor{lightblue}384                   & \cellcolor{lightblue}\textbf{0.351}        &\cellcolor{lightblue}\textbf{0.057}       \\
                         & Diff                & \textcolor{mygreen}{-0.8}                    & \textcolor{mygreen}{+0.17}        & \textcolor{mygreen}{+0.05 }       \\ \midrule
\multirow{3}{*}{YELP-Full}     
                         & MLP-50              & 38.7                  & 0.458        & 0.456       \\ 
                         & KAN-1 & 38.4 & 0.338 & 0.324\\
                         &\cellcolor{lightblue}FR-KAN-5               &\cellcolor{lightblue}38.4                  &\cellcolor{lightblue}\textbf{0.492}        &\cellcolor{lightblue}\textbf{0.475}       \\
                         & Diff                & \textcolor{mygreen}{-0.3}                     & \textcolor{mygreen}{+0.03}        &\textcolor{mygreen}{+0.02}       \\ \bottomrule
\end{tabular}%

        }
        \caption{Parameter count, accuracy, and F1 score of DistilBERT using different classification heads across all datasets. MLP-$x$ represents a 2-layer perceptron with $x$ hidden layer width. KAN-$x$ and FR-KAN-$x$ represent the corresponding 1-layer network with grid size $x$. The differences between the FR-KAN head and MLP have been shown.}
        \label{tab:performanceParameters}
      \end{varwidth}%
\end{table*}

\definecolor{lightblue}{rgb}{0.9, 0.95, 1.0}
\definecolor{mygreen}{rgb}{0.01, 0.5, 0.01}
\definecolor{myred}{rgb}{0.8, 0.01, 0.01}

\begin{table*}[ht]
\resizebox{\textwidth}{!}{%
\begin{tabular}{cccccccccccccccc}
\toprule
\multicolumn{2}{c}{\textbf{Method}} & \multicolumn{2}{c}{\textbf{AgNews}} & \multicolumn{2}{c}{\textbf{DBpedia}} & \multicolumn{2}{c}{\textbf{IMDb}} & \multicolumn{2}{c}{\textbf{Papluca}} & \multicolumn{2}{c}{\textbf{SST-5}} & \multicolumn{2}{c}{\textbf{TREC-50}} & \multicolumn{2}{c}{\textbf{YELP-Full}} \\ \cmidrule(lr){1-2} \cmidrule(lr){3-4} \cmidrule(lr){5-6} \cmidrule(lr){7-8} \cmidrule(lr){9-10} \cmidrule(lr){11-12} \cmidrule(lr){13-14} \cmidrule(lr){15-16}  
\textbf{Backbone} & \multicolumn{1}{c}{\textbf{Head}}                                 & \textbf{Acc \textcolor{mygreen}{$\uparrow$}}     & \textbf{F1 \textcolor{mygreen}{$\uparrow$}}      & \textbf{Acc \textcolor{mygreen}{$\uparrow$}}      & \textbf{F1 \textcolor{mygreen}{$\uparrow$}}      & \textbf{Acc \textcolor{mygreen}{$\uparrow$}}    & \textbf{F1 \textcolor{mygreen}{$\uparrow$}}     & \textbf{Acc \textcolor{mygreen}{$\uparrow$}}      & \textbf{F1 \textcolor{mygreen}{$\uparrow$}}      & \textbf{Acc \textcolor{mygreen}{$\uparrow$}}    & \textbf{F1 \textcolor{mygreen}{$\uparrow$}}    & \textbf{Acc \textcolor{mygreen}{$\uparrow$}}    & \textbf{F1 \textcolor{mygreen}{$\uparrow$}}     & \textbf{Acc \textcolor{mygreen}{$\uparrow$}}    & \textbf{F1 \textcolor{mygreen}{$\uparrow$}}     \\
\midrule
\multirow{4}{*}{BART}               & MLP            & 0.612            & 0.610            & 0.808             & 0.806            & \textbf{0.769}  & \textbf{0.769}  & 0.835             & 0.833            & \textbf{0.303}  & 0.181          & 0.254           & 0.042           & 0.364           & 0.357           \\
                                    & KAN            & 0.351            & 0.352            & 0.459             & 0.459            & 0.594           & 0.594           & 0.657             & 0.654            & 0.260           & 0.216          & 0.283           & 0.055           & 0.254           & 0.250           \\
                                    &\cellcolor{lightblue}FR-KAN         &\cellcolor{lightblue}\textbf{0.653}   &\cellcolor{lightblue}\textbf{0.651}   &\cellcolor{lightblue}\textbf{0.872}    &\cellcolor{lightblue}\textbf{0.872}   &\cellcolor{lightblue}0.749           &\cellcolor{lightblue}0.749           &\cellcolor{lightblue}\textbf{0.880}    &\cellcolor{lightblue}\textbf{0.879}   &\cellcolor{lightblue}0.273           &\cellcolor{lightblue}\textbf{0.237} &\cellcolor{lightblue}\textbf{0.451}  &\cellcolor{lightblue}\textbf{0.122}  &\cellcolor{lightblue}\textbf{0.390}  &\cellcolor{lightblue}\textbf{0.385}  \\ 
                                   & Diff & \textcolor{mygreen}{+0.04}  & \textcolor{mygreen}{+0.04}  & \textcolor{mygreen}{+0.06}  & \textcolor{mygreen}{+0.07}  & \textcolor{myred}{-0.02} & \textcolor{myred}{-0.02} & \textcolor{mygreen}{+0.05}  & \textcolor{mygreen}{+0.05}  & \textcolor{myred}{-0.03}  & \textcolor{mygreen}{+0.06} & \textcolor{mygreen}{+0.20}  & \textcolor{mygreen}{+0.08} & \textcolor{mygreen}{+0.03} & \textcolor{mygreen}{+0.03} \\
                                    \midrule
\multirow{4}{*}{BERT}               & MLP            & 0.772            & 0.770            & 0.752             & 0.746            & 0.763           & 0.762           & 0.579             & 0.581            & 0.374           & 0.197          & 0.181           & 0.008           & 0.440           & 0.433           \\
                                    & KAN            & 0.722            & 0.721            & 0.674             & 0.672            & 0.730           & 0.729           & 0.391             & 0.348            & 0.317           & 0.260          & 0.186           & 0.027           & 0.336           & 0.321           \\
                                    &\cellcolor{lightblue}FR-KAN         &\cellcolor{lightblue}\textbf{0.834}   &\cellcolor{lightblue}\textbf{0.833}   &\cellcolor{lightblue}\textbf{0.939}    &\cellcolor{lightblue}\textbf{0.938}   &\cellcolor{lightblue}\textbf{0.812}  &\cellcolor{lightblue}\textbf{0.811}  &\cellcolor{lightblue}\textbf{0.946}    &\cellcolor{lightblue}\textbf{0.945}   &\cellcolor{lightblue}\textbf{0.406}  &\cellcolor{lightblue}\textbf{0.339} &\cellcolor{lightblue}\textbf{0.378}  &\cellcolor{lightblue}\textbf{0.069}  &\cellcolor{lightblue}\textbf{0.471}  &\cellcolor{lightblue}\textbf{0.450}  \\
                                    & Diff & \textcolor{mygreen}{+0.06}  & \textcolor{mygreen}{+0.06}  & \textcolor{mygreen}{+0.19}  & \textcolor{mygreen}{+0.19}  & \textcolor{mygreen}{+0.05}  & \textcolor{mygreen}{+0.05}  & \textcolor{mygreen}{+0.37}  & \textcolor{mygreen}{+0.36}  & \textcolor{mygreen}{+0.03}  & \textcolor{mygreen}{+0.14} & \textcolor{mygreen}{+0.20}  & \textcolor{mygreen}{+0.06} & \textcolor{mygreen}{+0.03} & \textcolor{mygreen}{+0.02} \\
                                    \midrule
\multirow{4}{*}{DeBERTa}            & MLP            & 0.554            & 0.549            & 0.567             & 0.568            & 0.710           & 0.708           & 0.848             & 0.850            & \textbf{0.417}  & 0.289          & 0.199           & 0.015           & 0.393           & 0.380           \\
                                    & KAN            & 0.400            & 0.400            & 0.357             & 0.347            & 0.666           & 0.666           & 0.751             & 0.753            & 0.292           & 0.271          & 0.167           & 0.037           & 0.319           & 0.314           \\
                                    &\cellcolor{lightblue}FR-KAN         &\cellcolor{lightblue}\textbf{0.595}   &\cellcolor{lightblue}\textbf{0.594}   &\cellcolor{lightblue}\textbf{0.648}    &\cellcolor{lightblue}\textbf{0.648}   &\cellcolor{lightblue}\textbf{0.770}  &\cellcolor{lightblue}\textbf{0.770}  &\cellcolor{lightblue}\textbf{0.920}    &\cellcolor{lightblue}\textbf{0.923}   &\cellcolor{lightblue}0.367           &\cellcolor{lightblue}\textbf{0.333} &\cellcolor{lightblue}\textbf{0.377}  &\cellcolor{lightblue}\textbf{0.076}  &\cellcolor{lightblue}\textbf{0.412}  &\cellcolor{lightblue}\textbf{0.409}  \\ 
                                     & Diff & \textcolor{mygreen}{+0.04}  & \textcolor{mygreen}{+0.04}  & \textcolor{mygreen}{+0.08}  & \textcolor{mygreen}{+0.08}  & \textcolor{mygreen}{+0.06}  & \textcolor{mygreen}{+0.06}  & \textcolor{mygreen}{+0.07}  & \textcolor{mygreen}{+0.07}  & \textcolor{myred}{-0.05} & \textcolor{mygreen}{+0.04} & \textcolor{mygreen}{+0.18}  & \textcolor{mygreen}{+0.06} & \textcolor{mygreen}{+0.02} & \textcolor{mygreen}{+0.03} \\
                                    \midrule
\multirow{4}{*}{DistilBERT}         & MLP            & 0.836            & 0.834            & 0.865             & 0.863            & 0.795           & 0.794           & 0.878             & 0.877            & 0.352           & 0.184          & 0.096           & 0.009           & 0.429           & 0.420           \\
                                    & KAN            & 0.812            & 0.813            & 0.843             & 0.842            & 0.739           & 0.739           & 0.730             & 0.733            & 0.307           & 0.231          & 0.188           & 0.017           & 0.338           & 0.324           \\
                                    &\cellcolor{lightblue}FR-KAN         &\cellcolor{lightblue}\textbf{0.877}   &\cellcolor{lightblue}\textbf{0.876}   & \cellcolor{lightblue}\textbf{0.970}    &\cellcolor{lightblue}\textbf{0.971}   & \cellcolor{lightblue}\textbf{0.831}  &\cellcolor{lightblue}\textbf{0.830}  &\cellcolor{lightblue}\textbf{0.986}    &\cellcolor{lightblue}\textbf{0.986}   &\cellcolor{lightblue}\textbf{0.401}  &\cellcolor{lightblue}\textbf{0.336} &\cellcolor{lightblue}\textbf{0.351}  &\cellcolor{lightblue}\textbf{0.057}  &\cellcolor{lightblue}\textbf{0.492}  &\cellcolor{lightblue}\textbf{0.475} \\
                                     & Diff & \textcolor{mygreen}{+0.04}  & \textcolor{mygreen}{+0.04}  & \textcolor{mygreen}{+0.11}  & \textcolor{mygreen}{+0.11}  & \textcolor{mygreen}{+0.04}  & \textcolor{mygreen}{+0.04}  & \textcolor{mygreen}{+0.11}  & \textcolor{mygreen}{+0.11}  & \textcolor{mygreen}{+0.05}  & \textcolor{mygreen}{+0.15} & \textcolor{mygreen}{+0.26}  & \textcolor{mygreen}{+0.05} & \textcolor{mygreen}{+0.06} & \textcolor{mygreen}{+0.06} \\
                                    \midrule
\multirow{4}{*}{ELECTRA}            & MLP            & 0.480            & 0.471            & 0.364             & 0.338            & 0.592           & 0.590           & 0.539             & 0.509            & 0.326           & 0.175          & 0.096           & 0.009           & 0.291           & 0.284           \\
                                    & KAN            & 0.384            & 0.378            & 0.296             & 0.255            & 0.558           & 0.552           & 0.477             & 0.459            & 0.295           & 0.229          & 0.179           & 0.016           & 0.240           & 0.233           \\
                                    &\cellcolor{lightblue}FR-KAN         &\cellcolor{lightblue}\textbf{0.612}   &\cellcolor{lightblue}\textbf{0.606}   &\cellcolor{lightblue}\textbf{0.610}    &\cellcolor{lightblue}\textbf{0.609}   &\cellcolor{lightblue}\textbf{0.745}  &\cellcolor{lightblue}\textbf{0.745}  &\cellcolor{lightblue}\textbf{0.672}    &\cellcolor{lightblue}\textbf{0.670}   &\cellcolor{lightblue}\textbf{0.326}  &\cellcolor{lightblue}\textbf{0.285} &\cellcolor{lightblue}\textbf{0.338}  &\cellcolor{lightblue}\textbf{0.064}  &\cellcolor{lightblue}\textbf{0.370}  &\cellcolor{lightblue}\textbf{0.352}  \\
                                    & Diff & \textcolor{mygreen}{+0.13}  & \textcolor{mygreen}{+0.14}  & \textcolor{mygreen}{+0.25}  & \textcolor{mygreen}{+0.27}  &\textcolor{mygreen}{+0.15}  & \textcolor{mygreen}{+0.16}  & \textcolor{mygreen}{+0.13}  &\textcolor{mygreen}{+0.16}  & 0.00  & \textcolor{mygreen}{+0.11} & \textcolor{mygreen}{+0.24}  &\textcolor{mygreen}{+0.06} & \textcolor{mygreen}{+0.08} & \textcolor{mygreen}{+0.07} \\
                                    \midrule
\multirow{4}{*}{RoBERTa}            & MLP            & 0.420            & 0.304            & 0.258             & 0.155            & 0.578           & 0.546           & 0.327             & 0.226            & 0.269           & 0.085          & 0.132           & 0.008           & 0.209           & 0.088           \\
                                    & KAN            & 0.448            & 0.434            & 0.253             & 0.214            & 0.568           & 0.521           & 0.584             & 0.528            & 0.269           & 0.125          & 0.179           & 0.006           & 0.203           & 0.190           \\
                                    &\cellcolor{lightblue}FR-KAN         &\cellcolor{lightblue}\textbf{0.836}   &\cellcolor{lightblue}\textbf{0.832}   &\cellcolor{lightblue}\textbf{0.844}    &\cellcolor{lightblue}\textbf{0.829}   &\cellcolor{lightblue}\textbf{0.819}  &\cellcolor{lightblue}\textbf{0.819}  &\cellcolor{lightblue}\textbf{0.925}    &\cellcolor{lightblue}\textbf{0.910}   &\cellcolor{lightblue}\textbf{0.328}  &\cellcolor{lightblue}\textbf{0.189} &\cellcolor{lightblue}\textbf{0.179}  &\cellcolor{lightblue}\textbf{0.006}  &\cellcolor{lightblue}\textbf{0.369}  &\cellcolor{lightblue}\textbf{0.306}  \\
                                    & Diff & \textcolor{mygreen}{+0.42}  & \textcolor{mygreen}{+0.53}  & \textcolor{mygreen}{+0.59}  & \textcolor{mygreen}{+0.67}  & \textcolor{mygreen}{+0.24}  & \textcolor{mygreen}{+0.27}  & \textcolor{mygreen}{+0.60}  & \textcolor{mygreen}{+0.68}  & \textcolor{mygreen}{+0.06}  & \textcolor{mygreen}{+0.10} & \textcolor{mygreen}{+0.05}  & +0.00 & \textcolor{mygreen}{+0.16} & \textcolor{mygreen}{+0.22} \\\midrule
\multirow{4}{*}{XLNet}              & MLP            & \textbf{0.399}   & \textbf{0.375}   & \textbf{0.163}    & \textbf{0.156}   & \textbf{0.616}  & \textbf{0.613}  & \textbf{0.248}    & \textbf{0.195}   & 0.255           & 0.190          & \textbf{0.105}  & \textbf{0.019}  & 0.218           & 0.212           \\
                                    & KAN            & 0.275            & 0.261            & 0.123             & 0.112            & 0.536           & 0.528           & 0.166             & 0.153            & 0.225           & 0.211          & 0.102           & 0.015           & 0.207           & 0.203           \\
                                    &\cellcolor{lightblue}FR-KAN         & \cellcolor{lightblue}0.300            &\cellcolor{lightblue}0.293            &\cellcolor{lightblue}0.133             &\cellcolor{lightblue}0.124            &\cellcolor{lightblue}0.552           &\cellcolor{lightblue}0.533           &\cellcolor{lightblue}0.157             &\cellcolor{lightblue}0.138            &\cellcolor{lightblue}0.255           &\cellcolor{lightblue}\textbf{0.237} &\cellcolor{lightblue}0.048           &\cellcolor{lightblue}0.018           &\cellcolor{lightblue}\textbf{0.223}  &\cellcolor{lightblue}\textbf{0.221}  \\
                                     & Diff & \textcolor{myred}{-0.10} & \textcolor{myred}{-0.08} & \textcolor{myred}{-0.03} & \textcolor{myred}{-0.03} & \textcolor{myred}{-0.06 }& \textcolor{myred}{-0.08} & \textcolor{myred}{-0.09 }& \textcolor{myred}{-0.06} & {0.00}  & \textcolor{mygreen}{+0.05} & \textcolor{myred}{-0.06} & {0.00} & \textcolor{mygreen}{+0.01} & \textcolor{mygreen}{+0.01} \\\midrule
\multirow{4}{*}{Average}            & MLP            & 0.582            & 0.559            & 0.540             & 0.519            & 0.689           & 0.683           & 0.608             & 0.582            & 0.328           & 0.186          & 0.152           & 0.016           & 0.335           & 0.311           \\
                                    & KAN            & 0.485            & 0.480            & 0.429             & 0.414            & 0.627           & 0.618           & 0.537             & 0.518            & 0.281           & 0.220          & 0.183           & 0.025           & 0.271           & 0.262           \\
                                    &\cellcolor{lightblue}FR-KAN         &\cellcolor{lightblue}\textbf{0.672}   &\cellcolor{lightblue}\textbf{0.669}   &\cellcolor{lightblue}\textbf{0.717}    &\cellcolor{lightblue}\textbf{0.713}   &\cellcolor{lightblue}\textbf{0.754}  &\cellcolor{lightblue}\textbf{0.751}  &\cellcolor{lightblue}\textbf{0.784}    &\cellcolor{lightblue}\textbf{0.779}   &\cellcolor{lightblue}\textbf{0.337}  &\cellcolor{lightblue}\textbf{0.279} &\cellcolor{lightblue}\textbf{0.303}  &\cellcolor{lightblue}\textbf{0.059}  &\cellcolor{lightblue}\textbf{0.390}  &\cellcolor{lightblue}\textbf{0.371} \\
                                    & Diff &\textcolor{mygreen}{+0.09}  &\textcolor{mygreen}{+0.11}  &\textcolor{mygreen}{+0.18}  &\textcolor{mygreen}{+0.19}  &\textcolor{mygreen}{+0.07}  &\textcolor{mygreen}{+0.07}  &\textcolor{mygreen}{+0.18}  &\textcolor{mygreen}{+0.20}  &\textcolor{mygreen}{+0.01}  &\textcolor{mygreen}{+0.09} &\textcolor{mygreen}{+0.15}  &\textcolor{mygreen}{+0.04} &\textcolor{mygreen}{+0.05} &\textcolor{mygreen}{+0.06} \\\bottomrule
\end{tabular}%
}
\caption{Accuracy and F1 score of MLP, KAN, and FR-KAN classification heads on different backbones, evaluated across text classification datasets. FR-KAN consistently outperformed the other heads in both Accuracy and F1 score, with a few exceptions, such as for XLNet. The performance difference between the FR-KAN and the MLP head has been shown.}
\label{tab:performanceAccF1}
\end{table*}

\definecolor{lightblue}{rgb}{0.9, 0.95, 1.0}
\definecolor{mygreen}{rgb}{0.01, 0.5, 0.01}
\definecolor{myred}{rgb}{0.8, 0.01, 0.01}
\begin{table*}[ht]

\centering
\resizebox{\textwidth}{!}{%
\begin{tabular}{cccccccccccccccc}
\toprule
\multicolumn{2}{c}{\textbf{Method}} & \multicolumn{2}{c}{\textbf{AgNews}} & \multicolumn{2}{c}{\textbf{DBpedia}} & \multicolumn{2}{c}{\textbf{IMDb}} & \multicolumn{2}{c}{\textbf{Papluca}} & \multicolumn{2}{c}{\textbf{SST-5}} & \multicolumn{2}{c}{\textbf{TREC-50}} & \multicolumn{2}{c}{\textbf{YELP-Full}} \\ \cmidrule(lr){1-2} \cmidrule(lr){3-4} \cmidrule(lr){5-6} \cmidrule(lr){7-8} \cmidrule(lr){9-10} \cmidrule(lr){11-12} \cmidrule(lr){13-14} \cmidrule(lr){15-16} 
\textbf{Backbone} & \multicolumn{1}{c}{\textbf{Head}} &  \textbf{mF1} $\textcolor{mygreen}{\uparrow}$ &  $\kappa~ \textcolor{mygreen}{\uparrow}$ &  \textbf{mF1} $\textcolor{mygreen}{\uparrow}$ &  $\kappa~ \textcolor{mygreen}{\uparrow}$ &  \textbf{mF1} $\textcolor{mygreen}{\uparrow}$ &  $\kappa~ \textcolor{mygreen}{\uparrow}$ &  \textbf{mF1} $\textcolor{mygreen}{\uparrow}$ &  $\kappa~ \textcolor{mygreen}{\uparrow}$ &  \textbf{mF1} $\textcolor{mygreen}{\uparrow}$ &  $\kappa~ \textcolor{mygreen}{\uparrow}$ &  \textbf{mF1} $\textcolor{mygreen}{\uparrow}$ &  $\kappa~ \textcolor{mygreen}{\uparrow}$ &  \textbf{mF1} $\textcolor{mygreen}{\uparrow}$ &  $\kappa~ \textcolor{mygreen}{\uparrow}$\\ \midrule
\multirow{4}{*}{BART}           & MLP       & 0.612        & 0.483       & 0.808        & 0.793        & \textbf{0.769}       & \textbf{0.538}      & 0.835        & 0.826        & \textbf{0.303}       & 0.057       & 0.254        & 0.109        & 0.364         & 0.205         \\
                                & KAN       & 0.351        & 0.136       & 0.459        & 0.417        & 0.594       & 0.188      & 0.657        & 0.639        & 0.260       & 0.035       & 0.283        & 0.180        & 0.254         & 0.068         \\
                                & \cellcolor{lightblue} FR-KAN    & \cellcolor{lightblue}\textbf{0.653}        &\cellcolor{lightblue} \textbf{0.538}       &\cellcolor{lightblue} \textbf{0.872}        &\cellcolor{lightblue} \textbf{0.862 }       & \cellcolor{lightblue}0.749       & \cellcolor{lightblue}0.498      &\cellcolor{lightblue} \textbf{0.880}        &\cellcolor{lightblue}\textbf{0.873}        &\cellcolor{lightblue} 0.273       & \cellcolor{lightblue}\textbf{0.058}       & \cellcolor{lightblue}\textbf{0.451}        & \cellcolor{lightblue}\textbf{0.396}        &\cellcolor{lightblue}\textbf{0.390}         &\cellcolor{lightblue}\textbf{0.238}         \\ 
                                  & Diff & \textcolor{mygreen}{+0.04}  & \textcolor{mygreen}{+0.06}  & \textcolor{mygreen}{+0.06}  & \textcolor{mygreen}{+0.07}  & \textcolor{myred}{-0.02} & \textcolor{myred}{-0.04} & \textcolor{mygreen}{+0.04}  & \textcolor{mygreen}{+0.05}  & \textcolor{myred}{-0.03} & 0.00  & \textcolor{mygreen}{+0.20}  & \textcolor{mygreen}{+0.29}  & \textcolor{mygreen}{+0.03} & \textcolor{mygreen}{+0.03} \\                                
                                \midrule
\multirow{4}{*}{BERT}           & MLP       & 0.772        & 0.696       & 0.752        & 0.732        & 0.763       & 0.527      & 0.579        & 0.556        & 0.374       & 0.147       & 0.181        & 0.003        & 0.440         & 0.300         \\
                                & KAN       & 0.722        & 0.630       & 0.674        & 0.649        & 0.730       & 0.459      & 0.391        & 0.357        & 0.317       & 0.107       & 0.186        & 0.061        & 0.336         & 0.169         \\
                                & \cellcolor{lightblue}FR-KAN    &\cellcolor{lightblue}\textbf{0.834}        &\cellcolor{lightblue} \textbf{0.778}       & \cellcolor{lightblue}\textbf{0.939}        &\cellcolor{lightblue} \textbf{0.934}        &\cellcolor{lightblue} \textbf{0.812}       &\cellcolor{lightblue} \textbf{0.624}      &\cellcolor{lightblue} \textbf{0.946}        &\cellcolor{lightblue} \textbf{0.943}        &\cellcolor{lightblue} \textbf{0.406}       &\cellcolor{lightblue} \textbf{0.221}       &\cellcolor{lightblue} \textbf{0.378}        & \cellcolor{lightblue}\textbf{0.294}        & \cellcolor{lightblue}\textbf{0.471}         & \cellcolor{lightblue}\textbf{0.339}         \\ 
                                & Diff & \textcolor{mygreen}{+0.06}  & \textcolor{mygreen}{+0.08}  & \textcolor{mygreen}{+0.19}  & \textcolor{mygreen}{+0.20}  & \textcolor{mygreen}{+0.05}  & \textcolor{mygreen}{+0.10}  & \textcolor{mygreen}{+0.37}  & \textcolor{mygreen}{+0.39}  & \textcolor{mygreen}{+0.03}  & \textcolor{mygreen}{+0.07}  & \textcolor{mygreen}{+0.20}  & \textcolor{mygreen}{+0.29}  & \textcolor{mygreen}{+0.03} & \textcolor{mygreen}{+0.04} \\
                                \midrule
\multirow{4}{*}{DeBERTa}        & MLP       & 0.554        & 0.405       & 0.567        & 0.533        & 0.710       & 0.421      & 0.848        & 0.840        & \textbf{0.417}       & \textbf{0.216}       & 0.199        & 0.030        & 0.393         & 0.241         \\
                                & KAN       & 0.400        & 0.201       & 0.357        & 0.308        & 0.666       & 0.333      & 0.751        & 0.738        & 0.292       & 0.095       & 0.167        & 0.064        & 0.319         & 0.148         \\
                                & \cellcolor{lightblue}FR-KAN    & \cellcolor{lightblue}\textbf{0.595}        &\cellcolor{lightblue} \textbf{0.460}       & \cellcolor{lightblue}\textbf{0.648 }       & \cellcolor{lightblue}\textbf{0.621}        & \cellcolor{lightblue}\textbf{0.770}       & \cellcolor{lightblue}\textbf{0.539}      & \cellcolor{lightblue}\textbf{0.920}        &\cellcolor{lightblue} \textbf{0.916}        &\cellcolor{lightblue} 0.367       &\cellcolor{lightblue} 0.178       &\cellcolor{lightblue} \textbf{0.377}        &\cellcolor{lightblue} \textbf{0.294}        &\cellcolor{lightblue} \textbf{0.412}         &\cellcolor{lightblue} \textbf{0.265}         \\
                                & Diff & \textcolor{mygreen}{+0.04}  & \textcolor{mygreen}{+0.05}  & \textcolor{mygreen}{+0.08}  & \textcolor{mygreen}{+0.09}  & \textcolor{mygreen}{+0.06}  & \textcolor{mygreen}{+0.12}  &\textcolor{mygreen}{+0.07}  & \textcolor{mygreen}{+0.08}  & \textcolor{myred}{-0.05} & \textcolor{myred}{-0.04} & \textcolor{mygreen}{+0.18}  & \textcolor{mygreen}{+0.26}  & \textcolor{mygreen}{+0.02} & \textcolor{mygreen}{+0.02} \\
                                \midrule
\multirow{4}{*}{DistilBERT}     & MLP       & 0.836        & 0.781       & 0.865        & 0.854        & 0.795       & 0.591      & 0.878        & 0.871        & 0.352       & 0.117       & 0.096        & 0.019        & 0.429         & 0.287         \\
                                & KAN       & 0.812        & 0.750       & 0.843        & 0.831        & 0.739       & 0.477      & 0.730        & 0.716        & 0.307       & 0.081       & 0.188        & 0.033        & 0.338         & 0.172         \\
                                &\cellcolor{lightblue}FR-KAN    &\cellcolor{lightblue}\textbf{0.877}        &\cellcolor{lightblue}\textbf{0.836}       &\cellcolor{lightblue}\textbf{0.970}        &\cellcolor{lightblue}\textbf{0.968}        &\cellcolor{lightblue}\textbf{0.831}       &\cellcolor{lightblue}\textbf{0.662}      &\cellcolor{lightblue}\textbf{0.986}        &\cellcolor{lightblue}\textbf{0.985}        &\cellcolor{lightblue}\textbf{0.401}       &\cellcolor{lightblue}\textbf{0.208}       &\cellcolor{lightblue}\textbf{0.351}        &\cellcolor{lightblue}\textbf{0.258}        &\cellcolor{lightblue}\textbf{0.492}         &\cellcolor{lightblue}\textbf{0.365}         \\
                                & Diff & \textcolor{mygreen}{+0.04}  & \textcolor{mygreen}{+0.06}  & \textcolor{mygreen}{+0.11}  & \textcolor{mygreen}{+0.11}  & \textcolor{mygreen}{+0.04}  & \textcolor{mygreen}{+0.07}  & \textcolor{mygreen}{+0.11}  & \textcolor{mygreen}{+0.11}  & \textcolor{mygreen}{+0.05}  & \textcolor{mygreen}{+0.09}  & \textcolor{mygreen}{+0.25}  & \textcolor{mygreen}{+0.24}  & \textcolor{mygreen}{+0.06} & \textcolor{mygreen}{+0.08} \\
                                \midrule
\multirow{4}{*}{ELECTRA}        & MLP       & 0.480        & 0.306       & 0.364        & 0.315        & 0.592       & 0.186      & 0.539        & 0.514        & 0.326       & 0.082       & 0.096        & 0.007        & 0.291         & 0.114         \\
                                & KAN       & 0.384        & 0.178       & 0.296        & 0.242        & 0.558       & 0.118      & 0.477        & 0.450        & 0.295       & 0.069       & 0.179        & 0.020        & 0.240         & 0.051         \\
                                &\cellcolor{lightblue}FR-KAN    &\cellcolor{lightblue}\textbf{0.612}        &\cellcolor{lightblue}\textbf{0.482}       &\cellcolor{lightblue}\textbf{0.610}        &\cellcolor{lightblue}\textbf{0.580}        &\cellcolor{lightblue}\textbf{0.745}       &\cellcolor{lightblue}\textbf{0.490}      &\cellcolor{lightblue}\textbf{0.672}        &\cellcolor{lightblue}\textbf{0.655}        &\cellcolor{lightblue}0.326       &\cellcolor{lightblue}\textbf{0.124}       &\cellcolor{lightblue}\textbf{0.338}        &\cellcolor{lightblue}\textbf{0.249}        &\cellcolor{lightblue}\textbf{0.370}         &\cellcolor{lightblue}\textbf{0.213}         \\
                                 & Diff & \textcolor{mygreen}{+0.13}  & \textcolor{mygreen}{+0.18}  & \textcolor{mygreen}{+0.25}  & \textcolor{mygreen}{+0.26}  & \textcolor{mygreen}{+0.15}  & \textcolor{mygreen}{+0.30}  & \textcolor{mygreen}{+0.13}  & \textcolor{mygreen}{+0.14}  & 0.00  & \textcolor{mygreen}{+0.04}  & \textcolor{mygreen}{+0.24}  & \textcolor{mygreen}{+0.24}  & \textcolor{mygreen}{+0.08} & \textcolor{mygreen}{+0.10} \\
                                \midrule
\multirow{4}{*}{RoBERTa}        & MLP       & 0.420        & 0.224       & 0.258        & 0.197        & 0.578       & 0.151      & 0.327        & 0.292        & 0.269       & 0.000       & 0.132        & -0.008       & 0.209         & 0.013         \\
                                & KAN       & 0.448        & 0.263       & 0.253        & 0.193        & 0.568       & 0.129      & 0.584        & 0.561        & 0.269       & 0.005       & 0.179        & 0.000        & 0.203         & 0.002         \\
                                &\cellcolor{lightblue}FR-KAN    &\cellcolor{lightblue}\textbf{0.836}        &\cellcolor{lightblue}\textbf{0.781}       &\cellcolor{lightblue}\textbf{0.844}        &\cellcolor{lightblue}\textbf{0.832}        &\cellcolor{lightblue}\textbf{0.819}       &\cellcolor{lightblue}\textbf{0.639}      &\cellcolor{lightblue}\textbf{0.925}        &\cellcolor{lightblue}\textbf{0.921}        &\cellcolor{lightblue}\textbf{0.328}       &\cellcolor{lightblue}\textbf{0.089}       &\cellcolor{lightblue}\textbf{0.179}        &\cellcolor{lightblue}0.000        &\cellcolor{lightblue}\textbf{0.369}         &\cellcolor{lightblue}\textbf{0.212}         \\
                                & Diff & \textcolor{mygreen}{+0.42}  & \textcolor{mygreen}{+0.56}  & \textcolor{mygreen}{+0.59}  & \textcolor{mygreen}{+0.64}  & \textcolor{mygreen}{+0.24}  & \textcolor{mygreen}{+0.49}  & \textcolor{mygreen}{+0.60}  & \textcolor{mygreen}{+0.63}  & \textcolor{mygreen}{+0.06}  & \textcolor{mygreen}{+0.09}  & \textcolor{mygreen}{+0.05}  & \textcolor{mygreen}{+0.01}  & \textcolor{mygreen}{+0.16} & \textcolor{mygreen}{+0.20} \\
                                \midrule
\multirow{4}{*}{XLNet}          & MLP       & \textbf{0.399}        & \textbf{0.198}       & \textbf{0.163}        & \textbf{0.099}        & \textbf{0.616}       & \textbf{0.234}      & \textbf{0.248}        & \textbf{0.207}        & 0.255       & 0.011       & \textbf{0.105}        & \textbf{0.022}        & 0.218         & 0.021         \\
                                & KAN       & 0.275        & 0.035       & 0.123        & 0.055        & 0.536       & 0.070      & 0.166        & 0.122        & 0.225       & 0.013       & 0.102        & -0.001       & 0.207         & 0.008         \\
                                &\cellcolor{lightblue}FR-KAN    &\cellcolor{lightblue}0.300        &\cellcolor{lightblue}0.066       &\cellcolor{lightblue}0.133        &\cellcolor{lightblue}0.066        &\cellcolor{lightblue}0.552       &\cellcolor{lightblue}0.101      &\cellcolor{lightblue}0.157        &\cellcolor{lightblue}0.112        &\cellcolor{lightblue}0.255       &\cellcolor{lightblue}\textbf{0.048}       &\cellcolor{lightblue}0.048        &\cellcolor{lightblue}0.003        &\cellcolor{lightblue}\textbf{0.223}         &\cellcolor{lightblue}\textbf{0.029}         \\
                                & Diff & \textcolor{myred}{-0.10} & \textcolor{myred}{-0.13} & \textcolor{myred}{-0.03} & \textcolor{myred}{-0.03} & \textcolor{myred}{-0.06} & \textcolor{myred}{-0.13} & \textcolor{myred}{-0.09} & \textcolor{myred}{-0.09} & 0.00  & \textcolor{mygreen}{+0.04}  & \textcolor{myred}{-0.06} & \textcolor{myred}{-0.02} & 0.00 & \textcolor{mygreen}{+0.01} \\
                                \midrule
\multirow{4}{*}{Average}        & MLP       & 0.582        & 0.442       & 0.540        & 0.503        & 0.689       & 0.378      & 0.608        & 0.587        & 0.328       & 0.090       & 0.152        & 0.026        & 0.335         & 0.169         \\
                                & KAN       & 0.485        & 0.313       & 0.429        & 0.385        & 0.627       & 0.253      & 0.537        & 0.512        & 0.281       & 0.058       & 0.183        & 0.051        & 0.271         & 0.088         \\ 
                                &\cellcolor{lightblue}FR-KAN    &\cellcolor{lightblue}\textbf{0.673}        &\cellcolor{lightblue}\textbf{0.563}       &\cellcolor{lightblue}\textbf{0.717}        &\cellcolor{lightblue}\textbf{0.695}        &\cellcolor{lightblue}\textbf{0.754}       &\cellcolor{lightblue}\textbf{0.508}      &\cellcolor{lightblue}\textbf{0.784}        &\cellcolor{lightblue}\textbf{0.772}        &\cellcolor{lightblue}\textbf{0.337}       &\cellcolor{lightblue}\textbf{0.132}       &\cellcolor{lightblue}\textbf{0.303}        &\cellcolor{lightblue}\textbf{0.213}        &\cellcolor{lightblue}\textbf{0.390}         &\cellcolor{lightblue}\textbf{0.237}     \\
                                & Diff & \textcolor{mygreen}{+0.09}  & \textcolor{mygreen}{+0.12}  & \textcolor{mygreen}{+0.18}  & \textcolor{mygreen}{+0.19}  & \textcolor{mygreen}{+0.06}  & \textcolor{mygreen}{+0.13}  & \textcolor{mygreen}{+0.18}  & \textcolor{mygreen}{+0.19}  & \textcolor{mygreen}{+0.01}  & \textcolor{mygreen}{+0.04}  & \textcolor{mygreen}{+0.15}  & \textcolor{mygreen}{+0.19}  & \textcolor{mygreen}{+0.05} & \textcolor{mygreen}{+0.07} \\
                                \bottomrule   
\end{tabular}%
} 
\caption{Micro F1 and Kappa score of MLP, KAN, and FR-KAN classification heads on different backbones, evaluated across text classification datasets. Similar to Table \ref{tab:performanceAccF1}, the FR-KAN head consistently outperforms the MLP and KAN heads.
}
\label{tab:performanceMicroKappa}
\end{table*}

\section{Result Analysis}

\subsection{Efficacy of FR-KAN head fine-tuning}
\label{sec:exp1}
The FR-KAN head significantly outperformed the MLP (Eq. \ref{eq:mlp1}) and KAN (Eq. \ref{eq:kan} and \ref{eq:phiBSpline}) classification heads across most dataset-model pairs (Tab. \ref{tab:performanceAccF1} and Appendix Tab. \ref{tab:performanceMicroKappa}), while requiring similar time to fine-tune. Among the models, RoBERTa showed the largest improvement with an average increase of $0.3$ in accuracy and $0.35$ in F1 score over the MLP heads. The only exception was XLNet, where FR-KAN classifiers performed similarly or slightly worse compared to the MLPs, with an average decrease of $0.05$ in accuracy and $0.03$ in F1 score. Nonetheless, across all the datasets, the FR-KAN classifiers substantially outperformed both MLPs and KANs with an average increase of $10\%$ in accuracy and $11\%$ in F1 score over MLPs.

\begin{figure}[ht]
    \centering
        \begin{subfigure}{0.32\linewidth}
        \centering
        \includegraphics[width=\linewidth]{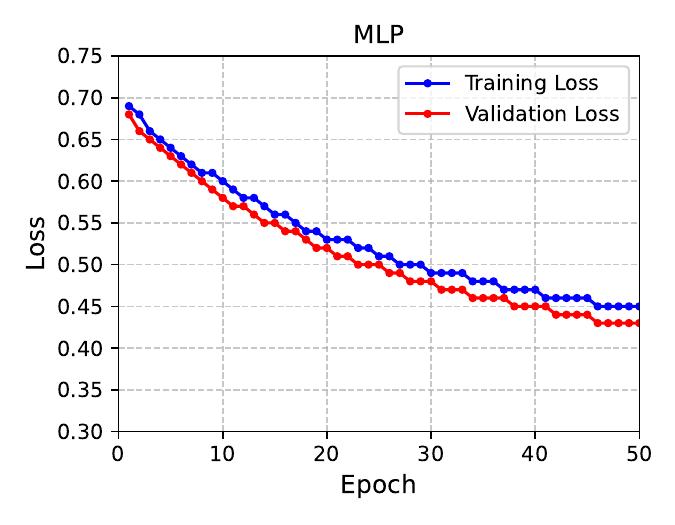}
        \caption{}
        \label{fig:MLPLoss}
    \end{subfigure} \hfill
    \begin{subfigure}{0.32\linewidth}
        \centering
        \includegraphics[width=\linewidth]{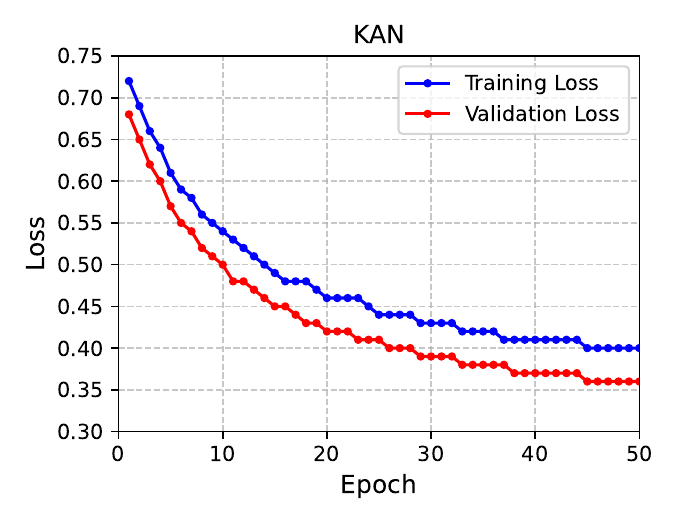}
        \caption{}
        \label{fig:KANLoss}
    \end{subfigure}\hfill
    \begin{subfigure}{0.32\linewidth}
        \centering
        \includegraphics[width=\linewidth]{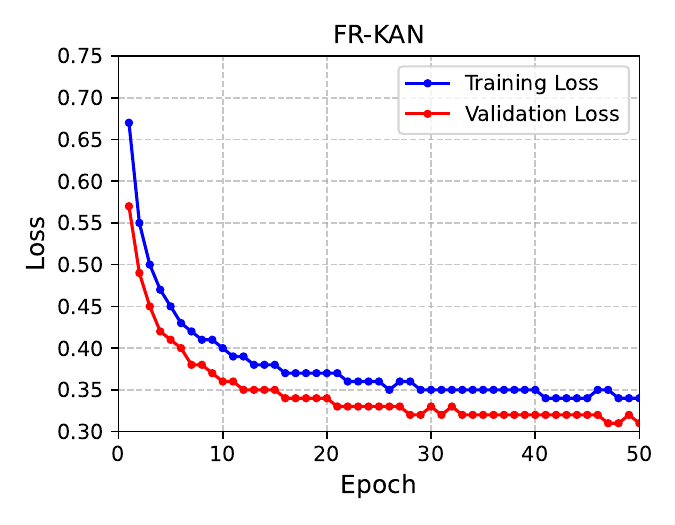}
        \caption{}
        \label{fig:frKANLoss}
    \end{subfigure}
    \begin{subfigure}{0.32\linewidth}
        \centering
        \includegraphics[width=\linewidth]{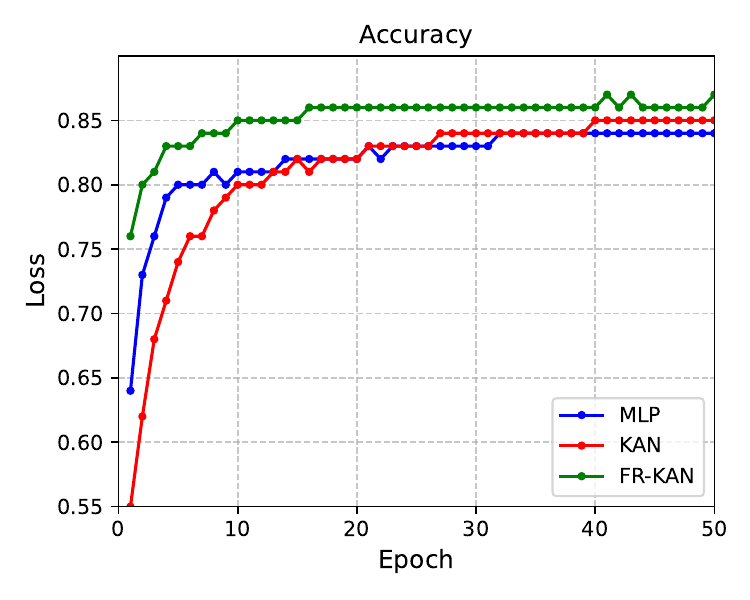}
        \caption{}
        \label{fig:acuracy}
    \end{subfigure}\hfill
    \begin{subfigure}{0.32\linewidth}
        \centering
        \includegraphics[width=\linewidth]{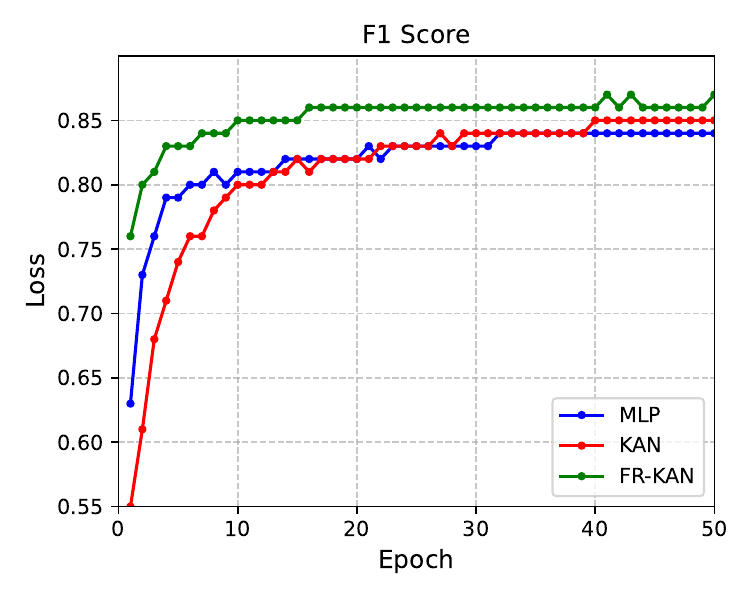}
        \caption{}
        \label{fig:f1}
    \end{subfigure}\hfill
    \begin{subfigure}{0.32\linewidth}
        \centering
        \includegraphics[width=\linewidth]{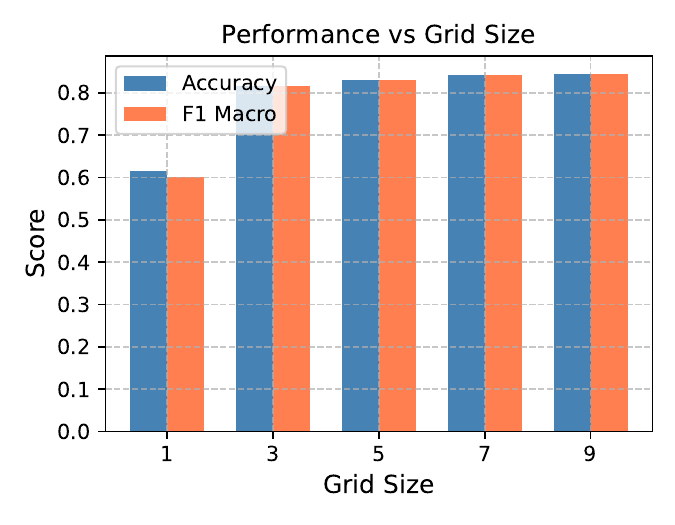}
        \caption{}
        \label{fig:grid}
    \end{subfigure}    
    \caption{Results of the DistilBERT model on the IMDb dataset. For different classification heads, (a)-(c) training and validation loss, (d) accuracy, and (e) F1 score. For the FR-KAN head, (f) accuracy and F1 score at varying grid sizes.}
\end{figure}

\subsection{Convergence of FR-KAN head}
Following the training and validation loss, the FR-KAN head (Fig. \ref{fig:frKANLoss}) converges significantly faster than the other two heads (Fig. \ref{fig:MLPLoss} and \ref{fig:KANLoss}) when trained for $50$ epochs. Additionally, Fig. \ref{fig:acuracy} and \ref{fig:f1} show a substantial increase in both accuracy and F1 score at every training epoch. For instance - the accuracy of the FR-KAN head in the $4^{th}$ epoch is achieved by the other heads after fine-tuning for more than $20$ epochs. The faster convergence of FR-KAN, especially in comparison to spline-based KAN, can be attributed to the smoother functional representation by utilizing the Fourier series.

\subsection{Parameter Efficiency of FR-KAN head}
\label{sec:exp2}
We adjust the trainable parameter count of the MLP classifier by varying the hidden layer width of the MLP defined in Eq. \ref{eq:mlp2}, to match or exceed the trainable parameter count of the FR-KAN head. The KAN heads with a grid size of $1$ have the same number of trainable parameters as the FR-KAN heads with a grid size of $5$. As observed in Tab. \ref{tab:performanceParameters}, the FR-KAN heads not only outperform MLP and KAN heads in terms of performance but do so while requiring equal or fewer parameters.

\subsection{Impact of Grid Size on Performance}
The grid size signifies how fine-grained the FR-KAN coefficients will be in a single FR-KAN layer. Aligning with Corollary \ref{cr:gridError}, increasing the grid size leads to performance improvement, though, the gains diminish at higher grid sizes due to convergence, as evident from Fig. \ref{fig:grid}. Models may also experience overfitting at larger grid sizes, as shown in Tab. \ref{tab:gridAblation}, where a slight drop in performance is observed at higher grid sizes.



\section{Discussion}

\subsection{B-Splines vs Fourier Series}
Both the B-splines of the original KAN implementation and the Fourier series of the FR-KAN head are used to approximate continuous univariate functions. Both methods work well with smooth, continuous, and low-dimensional functions while being computationally inexpensive. At points of discontinuities, splines require higher degree functions while Fourier series suffer from oscillations due to the Gibbs phenomenon.

The Fourier series has several key advantages over splines. By nature, the Fourier representation uses sines and cosines instead of the piece-wise polynomials in splines. Hence, the Fourier series can represent smoother periodic functions which can be advantageous as smoothness can improve KAN performance \cite{samadi2024smooth}. The Fourier series also has better global control compared to the better local control of splines which can contribute to improved parameter efficiency in certain tasks. However, one key downside of the Fourier series is the decline of interpretability in comparison to splines -- although both methods are significantly more interpretable than perceptrons.

\subsection{Evaluation Fairness}
Throughout the work, we gave utmost importance to ensuring that all three classification heads were evaluated on equal grounds. This involved using the same backbone architecture, dataset splits, and training configuration across all experiments. We further attempted to evaluate the classifiers with an equal number of trainable parameters to mitigate biases induced by training larger models. The experiments were conducted using multiple seed values in the same hardware configuration and the average result had been taken. We concur that the results might differ based on the pre-trained model weights, dataset splits, and training configuration but the differences are expected to be negligible. 

\section{Broader Impact}
\subsection{Greener Approach}
Pre-training and fully fine-tuning large networks can be attributed to high energy consumption and substantial carbon footprint \cite{patterson2021carbon}. Classification head fine-tuning offers a greener alternative, achieving slightly worse or better \cite{kumar2022fine} performance compared to fully fine-tuning. These transfer learning strategies also ensure the reusability of pre-trained weights, consequently avoiding redundant pre-training. The proposed FR-KAN heads train faster than the standard MLP heads, thereby consuming less resources and leaving less carbon footprint. 

\subsection{Universal MLP Alternative}
The promising empirical results of FR-KAN heads in the domain of text classification affirms their potential as a generalized MLP alternative. We envision that FR-KANs will not be limited to classification heads only and can be incorporated as neural network layers, \emph{e.g} within the transformer architecture. The adaptability can be explored in other domains \emph{e.g} computer vision, time series analysis, and speech analysis. 
\begin{table*}[htpb]
\centering
\centering
\resizebox{.99\textwidth}{!}{%
\begin{tabular}{cccccccccccc}
\toprule 
\multicolumn{2}{c}{\textbf{Grid Size}}                             & \multicolumn{2}{c}{\textbf{1}} & \multicolumn{2}{c}{\textbf{2}} & \multicolumn{2}{c}{\textbf{3}} & \multicolumn{2}{c}{\textbf{4}} & \multicolumn{2}{c}{\textbf{5}} \\ \cmidrule(lr){1-2} \cmidrule(lr){3-4} \cmidrule(lr){5-6}\cmidrule(lr){7-8} \cmidrule(lr){9-10} \cmidrule(lr){11-12}
\multicolumn{1}{c}{\textbf{Dataset}} & \multicolumn{1}{c}{\textbf{Method}}             & \textbf{Acc}    & \textbf{F1}   & \textbf{Acc}    & \textbf{F1}   & \textbf{Acc}    & \textbf{F1}   & \textbf{Acc}    & \textbf{F1}   & \textbf{Acc}   & \textbf{F1}   \\ \midrule
\multicolumn{1}{c}{AgNews}           & DistilBERT & 0.812           & 0.813         & 0.832           & 0.829         & 0.852           & 0.852         & 0.864           & 0.862         & 0.877          & 0.876         \\ 
\multicolumn{1}{c}{DBpedia}          &   DistilBERT                          & 0.739           & 0.710         & 0.892           & 0.890         & 0.951           & 0.951         & 0.961           & 0.961         & 0.970          & 0.971         \\ 
\multicolumn{1}{c}{IMDb}             & DistilBERT                             & 0.616           & 0.601         & 0.808           & 0.808         & 0.817           & 0.817         & 0.830           & 0.830         & 0.831          & 0.830         \\ 
\multicolumn{1}{c}{Papluca}          & DistilBERT                            & 0.768           & 0.754         & 0.933           & 0.934         & 0.976           & 0.977         & 0.984           & 0.983         & 0.986          & 0.986         \\ 
\multicolumn{1}{c}{SST-5}              & BERT                        & 0.310           & 0.197         & 0.376           & 0.240         & \underline{0.396}           & 0.316         & 0.394           & 0.330         & 0.406          & 0.339         \\ 
\multicolumn{1}{c}{TREC-50}             & BART                        & 0.343           & 0.046         & 0.412           & 0.086         & 0.430           & 0.096         & \underline{0.467}           & \underline{0.125}         & 0.451          & 0.122         \\ 
\multicolumn{1}{c}{YELP-Full}             & DistilBERT                  & 0.385           & 0.328         & 0.420           & 0.401         & 0.457           & 0.439         & 0.466           & 0.453         & 0.492          & 0.475         \\ \bottomrule
\end{tabular}
}
\caption{Change in accuracy and F1 score with grid size for the best performing FR-KAN models from Tab. \ref{tab:performanceAccF1}. While the performance of most models improves with the increase in grid size, a few cases where a smaller grid size outperforms the larger grid size are \underline{underlined}. The grid size of 5 usually shows the best performance for all the benchmarks and has been chosen as the default grid size for all experiments.}
\label{tab:gridAblation}
\end{table*}
\section{Related Work}
\label{sec:relatedWork}
\paragraph{Multi-layer Perceptrons (MLPs)} Inspired by the neurons in the biological brain, the original unilayer perceptron dates back to the 50s and was intended as a machine for pattern recognition \cite{rosenblatt1958perceptron}. Initially constrained by its single-layer architecture, the methodology was later expanded to multi-layer perceptrons based on the universal approximation theorem which states that a continuous function can be approximated by a feedforward network with a finite number of neurons \cite{hornik1989multilayerMLP}. During the deep learning era, researchers quickly adopted MLPs as a fundamental module in several foundational deep learning architectures \cite{lecun1998gradient, krizhevsky2017imagenet}. The versatility and adaptability of MLPs have made them the most popular choice for classification heads in virtually all classification-based tasks \cite{abiodun2018state}.

\paragraph{Kolmogorov-Arnold Networks (KANs)}
KAN \cite{liu2024kan} is based on the Kolmogorov-Arnold representation theorem \cite{tikhomirov1991representation} stating that a continuous multivariate function is a composition of multiple continuous univariate functions and addition operations. Boasting faster neural scaling laws and interpretability KANs became popular in multiple domains including time series analysis \cite{vaca2024kolmogorov} and forecasting \cite{xu2024kolmogorov, genet2024temporal}, satellite image classification \cite{cheon2024kolmogorov}, mechanics problems \cite{abueidda2024deepokan}, and quantum architecture search \cite{kundu2024kanqas}. KANs have developed multiple variations utilizing the wavelet transform \cite{bozorgasl2024wav}, Jacobi basis functions \cite{aghaei2024fkan}, radial basis functions \cite{li2024kolmogorov}, and several functional combinations \cite{ta2024fc,ta2024bsrbf}. We highlight the work of \cite{xu2024fourierkan}, where the Fourier KAN was introduced to improve graph collaborative filtering in recommendation tasks.   
\paragraph{Transformers in Text Classification}
The transformer architecture \cite{vaswani2017attention}, initially introduced for sequence-to-sequence generation, has revolutionized various domains of NLP, including text classification \cite{cunha2023comparative}. While primarily proposed as encoder-decoder architecture, pre-trained transformer encoders, such as BERT \cite{devlin2018bert} and its variants \cite{sanh2019distilbert, liu2019roberta, he2020deberta} have been popular in natural language sequence classification tasks.
\paragraph{Linear Probing}
Full fine-tuning demands significant computational resources and is susceptible to overfitting on smaller datasets \cite{lian2022scaling,gao2023tuning}. In contrast, finetuning the classification head, \emph{i.e.} linear probing, can be a resource-efficient alternative with enhanced robustness on out-of-distribution data \cite{kumar2022fine}. Linear probing can be enhanced via several strategies, \emph{e.g.} parameter-efficient tuning \cite{yang2022parameter}. 


\section{Conclusion}
Our work explores Fourier-KAN as a promising alternative to MLPs for text classification using pre-trained transformer classifiers. We find that FR-KANs perform better, require fewer parameters, and train faster compared to MLPs. In the future, we wish to investigate the potential of FR-KANs replacing MLPs inside the transformer architectures.

\section*{Limitations}
Although significant strides have been made by the FR-KAN head over its KAN predecessor, the improved performance comes at the expense of interpretability. We were also unable to evaluate the classification heads with a lower parameter count, as the number of parameters in KAN heads is constrained by the grid size, which cannot be reduced beyond the minimum value of $1$. Finally, our study is limited to the use of B-splines and Fourier series as the univariate function in the KAN representation. Alternative methods of function approximation might produce interesting results.

\bibliography{ref}

\end{document}